\documentclass{article}

\usepackage{times}
\usepackage{graphicx} 

\usepackage{natbib}

\usepackage{algorithm}
\usepackage{algorithmic}

\usepackage{amssymb,amsthm}
\newtheorem{theorem}{Theorem}
\newtheorem{lemma}[theorem]{Lemma}

\usepackage{hyperref}

\usepackage[accepted]{icml2015_arxiv}

\icmltitlerunning{Learning Privately from Multiparty Data}

\begin{document} 

\twocolumn[
\icmltitle{Learning Privately from Multiparty Data} 

\icmlauthor{Jihun Hamm}{hammj@cse.ohio-state.edu}
\icmladdress{The Ohio State University, 395 Dreese Lab, 2015 Neil Ave,
            Columbus, OH 43210 USA}
\icmlauthor{Paul Cao}{yic242@eng.ucsd.edu}
\icmladdress{University of California, San Diego, 9500 Gilman Drive, 
La Jolla, CA 92093 USA}
\icmlauthor{Mikhail Belkin}{mbelkin@cse.ohio-state.edu}
\icmladdress{The Ohio State University, 395 Dreese Lab, 2015 Neil Ave,
            Columbus, OH 43210 USA}

\vskip 0.3in
]

\begin{abstract} 
Learning a classifier from private data collected by multiple parties 
is an important problem that has many potential applications.
How can we build an accurate and differentially private global classifier 
by combining locally-trained classifiers from different parties, without
access to any party's private data?
We propose to transfer the `knowledge' of the local classifier ensemble
by first creating labeled data from auxiliary unlabeled data,
and then train a global $\epsilon$-differentially private classifier.
We show that majority voting is too sensitive and therefore propose 
a new risk weighted by class probabilities estimated from the ensemble. 
Relative to a non-private solution, our private solution has a generalization 
error bounded by $O(\epsilon^{-2}M^{-2})$ where $M$ is the number of parties.
This allows strong privacy without performance loss 
when $M$ is large, such as in crowdsensing applications. 
We demonstrate the performance of our method with realistic tasks 
of activity recognition, 
network intrusion detection, and malicious URL detection. 
\end{abstract} 

\section{Introduction}
 
Consider the problem of performing machine learning with data collected
by multiple parties. 
In many settings, the parties may not wish to disclose the private information. 
For example, the parties can be medical institutions who aim to perform collaborative
research using sensitive patient information they hold. 
For another example, the parties can be computer users who aim to collectively build
a malware detector without sharing their usage data.
A conventional approach to learning from multiparty data is to first collect
data from all parties and then process them centrally.
When privacy is a major concern, this approach is not always an appropriate
solution since it is vulnerable to attacks during transmission,
storage, and processing of data.
Instead, we will consider a setting in which each party trains a {\it local} 
classifier from its private data without sending the data.
The goal is to build a {\it global} classifier
by combining local classifiers efficiently and privately. 
We expect the global classifier to be more accurate than individual local
classifiers, as it has access to more information than individual classifiers.

This problem of aggregating classifiers was considered in \cite{Pathak:2010},
where the authors proposed averaging of the {\it parameters} of local classifiers
to get a global classifier.
To prevent the leak of private information from the averaged parameters,
the authors used a differentially private mechanism.
Differential privacy measures maximal change in the probability of any outcome of 
a procedure when any item is added to or removed from a database.
It provides a strict upper bound on the privacy loss against any adversary \cite{Dwork:2004:CRYPTO,Dwork:2006:TC,Dwork:2006:ALP}.
Parameter averaging is a simple and practical procedure that can be implemented
by Secure Multiparty Computation \cite{Yao:1982:FOCS}. 
However, averaging is not applicable to classifiers with non-numerical
parameters such as decision trees, nor to a collection of different classifier types.
This raises the question if there are more flexible and perhaps better 
ways of aggregating local classifiers privately.

\begin{figure*}[th]
\centering
\includegraphics[width=0.99\linewidth]{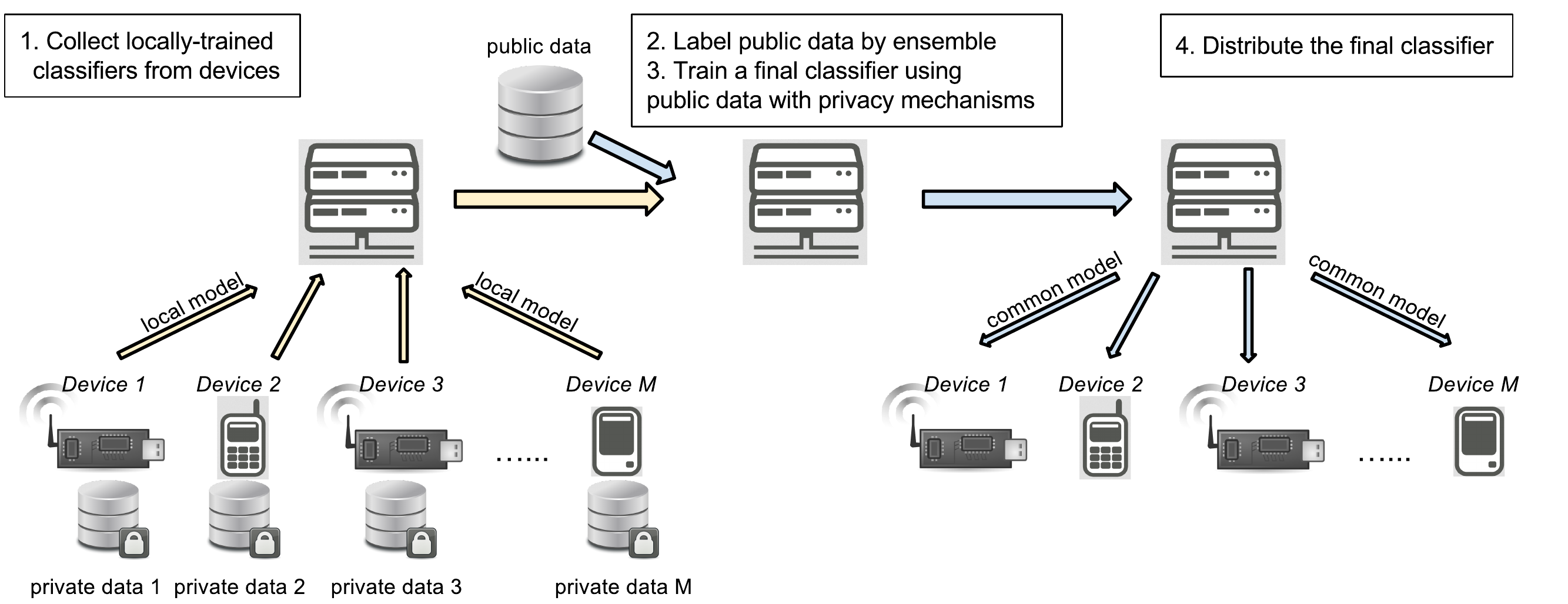}
\caption{Workflow of the proposed algorithm. 
Assume that the parties are smart devices. Each party holds a small amount of
private data and uses the data to train a local classifier. The ensemble of local
classifiers collected then generates labels for auxiliary data,
which in turn are used for training a global classifier. The final classifier is
released after sanitization for privacy.
}
\label{fig:system}
\end{figure*}

In this paper, we propose a method of building a global differentially private 
classifier from an ensemble of local classifier in two steps (see Fig.~\ref{fig:system}.)
In the first step, locally-trained classifiers are collected by a trusted entity.
A naive approach to use the collected classifiers is to release (the parameters of)
the classifiers after sanitization by differentially private mechanisms,
which is impractical (Sec.~\ref{sec:direct release}.)
Instead, we use the classifier ensemble to generates (pseudo)labels for 
auxiliary {\it unlabeled} data, thus transferring the knowledge of the ensemble
to the auxiliary data. 
In the second step, we use the labeled auxiliary data to find an empirical risk
minimizer, and release a differentially private classifier using output perturbation
\cite{Chaudhuri:2011:JMLR}. 

When generating labels for auxiliary data using an ensemble of classifiers, 
majority voting is the simplest choice. 
However, we show quantitatively that a global classifier trained from majority-voted
labels is highly sensitive to individual votes of local classifiers.
Consequently, the final classifier after differentially-private
sanitization suffer a heavy loss in its performance.
To address this, we propose a new risk insensitive to individual votes, 
where each sample is {\it weighted} by the confidence of the ensemble.
We provide an interpretation of the weighted risk in terms of random
hypothesis of an ensemble \cite{Breiman:1996} in contrast to deterministic
labeling rule of majority voting.
One of our main results is in Theorem~\ref{thm:performance 2}:
we can achieve $\epsilon$-differential privacy with a generalization
error of $O(\epsilon^{-2}M^{-2})$ and $O(N^{-1})$ terms, relative to the
expected risk of a non-private solution, 
where $M$ is the number of parties and $N$ is the number of samples in auxiliary data.
This result is especially useful in a scenario where there are a large number of
parties with weak local classifiers such as a group of connected smart devices with
limited computing capability. 
We demonstrate the performance of our approach with several realistic tasks: 
activity recognition,
network intrusion detection, and malicious URL detection. 
The results show that it is feasible to achieve both accuracy and
privacy with a large number of parties.

To summarize, we propose a method of building a global differentially
private classifier from locally-trained classifiers of multiple parties without access
to their private data. The proposed method has the following advantages: 
1) it can use local classifiers of any {(mixed)} types and therefore is flexible;
2) its generalization error converges to that of a non-private solution
with a fast rate of $O(\epsilon^{-2}M^{-2})$ and $O(N^{-1})$;
3) it also provides $\epsilon$-differential privacy to all samples of a party and
not just a single sample.

In Sec.~\ref{sec:background}, we formally describe privacy definitions.
In Sec.~\ref{sec:step 1}, we discuss the first step of leveraging unlabeled data,
and in Sec.~\ref{sec:step 2}, we present the second step of finding
a global private classifier via empirical risk minimization in two different forms. 
In Sec.~\ref{sec:related work}, we discuss related works.
We evaluate the methods with real tasks in Sec.~\ref{sec:experiments} and conclude
the paper in Sec.~\ref{sec:conclusion}.
Appendix contains omitted proofs and several extensions of 
the algorithms.

\section{Preliminary} \label{sec:background}


\subsection{Differential privacy}
A randomized algorithm that takes data $\mathcal{D}$ as input 
and outputs a function $f$ is called $\epsilon$-differentially private if
\begin{equation} \label{eq:differential privacy}
\frac{P(f(\mathcal{D}) \in \mathcal{S})}
{P(f(\mathcal{D}') \in \mathcal{S})} \leq e^{\epsilon}
\end{equation}
for all measurable $\mathcal{S} \subset \mathcal{T}$ of the output range
and for all datasets $\mathcal{D}$ and $\mathcal{D}'$ differing in a single item,
denoted by $\mathcal{D} \sim \mathcal{D}'$.
That is, even if an adversary knows the whole dataset $\mathcal{D}$ except for
a single item,
she cannot infer much about the unknown item from the output $f$ of the algorithm.
When an algorithm outputs a real-valued vector $f \in \mathbb{R}^D$, its global
$L_2$ sensitivity \cite{Dwork:2006:TC} can be defined as 
\begin{equation}\label{eq:sensitivity}
S(f) = \max_{\mathcal{D}\sim\mathcal{D}'} \|f(\mathcal{D}) - f(\mathcal{D}')\|
\end{equation}
where $\|\cdot\|$ is the $L_2$ norm.
An important result from \cite{Dwork:2006:TC} 
is that a vector-valued output $f$ with sensitivity $S(f)$ can be made 
$\epsilon$-differentially private 
by perturbing $f$ with an additive noise vector $\eta$
whose density is
\begin{equation}\label{eq:sensitivity method}
P(\eta) \propto e^{-\frac{\epsilon}{S(f)}\|\eta\|}.
\end{equation} 

\subsection{Output perturbation}

When a classifier which minimizes empirical risk is released in public, 
it leaks information about the training data. 
Such a classifier can be sanitized by perturbation with additive noise
calibrated to the sensitivity of the classifier, known as output perturbation method \cite{Chaudhuri:2011:JMLR}.
Specifically, the authors show the following.
If $w_s$ is the minimizer of the regularized empirical risk
\begin{equation}
R^\lambda_S(w) = \frac{1}{N} 
\sum_{(x,y) \in S} l(h(x;w),y) + \frac{\lambda}{2}\|w\|^2,
\end{equation}
Then the perturbed output 
$w_p = w_s + \eta,\;
p(\eta) \propto e^{-\frac{N\lambda \epsilon}{2}\|\eta\|}$
is $\epsilon$-differentially private for a single sample.
Output perturbation was used to sanitize the averaged parameters in \cite{Pathak:2010}.
We also use output perturbation to sanitize global classifiers. 
One important difference of our setting to previous work is that
we consider $\epsilon$-differential privacy of all samples of a party, 
which is much stronger than $\epsilon$-differential privacy of a single sample. 

There are conditions on the loss function for this guarantee to hold.
We will assume the following conditions for our global classifier\footnote{Local classifiers are allowed to be of any type.}
similar to \cite{Chaudhuri:2011:JMLR}.
\begin{itemize}\setlength{\itemsep}{-2pt}
\item The loss-hypothesis has a form $l(h(x;w),v)) = l(vw^T\phi(x))$, 
where $\phi:\mathcal{X} \to \mathbb{R}^d$ is a fixed map.
We will consider only linear classifiers $l(vw^Tx)$, where any nonlinear map 
$\phi$ is absorbed into the $d$-dimensional features. 
\item The surrogate loss $l(\cdot)$ is convex and continuously differentiable.
\item The derivative $l'(\cdot)$ is bounded: 
$|l'(t)| \leq 1,\; \forall t \in \mathbb{R}$, 
and $c$-Lipschitz: 
$|l'(s)-l'(t)| \leq c |s-t|,\;\forall s,t \in \mathbb{R}.$
\item The features are bounded: $\sup_{x\in \mathcal{X}} \|x\| \leq 1$.
\end{itemize}
These conditions are satisfied by, e.g., logistic regression loss ($c=1/4$)
and approximate hinge loss. 

\section{Transferring knowledge of ensemble} \label{sec:step 1}

\subsection{Local classifiers}

In this paper, we treat local classifiers as $M$ black boxes
$h_1(x),..., h_M(x)$.
We assume that a local classifier $h_i(x)$ is trained using its private i.i.d.
training set $S^{(i)}$  
\begin{equation}\label{eq:private data}
S^{(i)} = \{(x^{(i)}_1,y^{(i)}_1),\;...\;,(x^{(i)}_{N_i},y^{(i)}_{N_i})\},
\end{equation}
where $(x^{(i)}_j,y^{(i)}_j) \in \mathcal{X}\times\{-1,1\}$ is a sample from
a distribution $P(x,y)$ common to all parties.
We consider binary labels $y\in\{-1,1\}$ in the main paper, and present
a multiclass extension in Appendix~\ref{sec:appendix b}.

This splitting of training data across parties is similar to the 
Bagging procedure \cite{Breiman:1996} with some differences.
In Bagging, the training set $S^{(i)}$ for party $i$ is sampled {\it with replacement}
from the whole dataset $\mathcal{D}$, whereas in our setting, the training set
is sampled {\it without replacement} from $\mathcal{D}$,
more similar to the Subagging \cite{Politis:1999} procedure.

\subsection{Privacy issues of direct release}\label{sec:direct release}

In the first step of our method, local classifiers from multiple parties 
are first collected by a trusted entity. 
A naive approach to use the ensemble is to directly release
the local classifier parameters to the parties after appropriate sanitization.
However, this is problematic in efficiency and privacy.
Releasing all $M$ classifier parameters is an operation with a constant sensitivity,
as opposed to releasing insensitive statistics such as an average whose sensitivity
is $O(M^{-1})$.
Releasing the classifiers requires much stronger perturbation than necessary, 
incurring steep loss of performance of sanitized classifiers.
Besides, efficient differentially private mechanisms are known only for
certain types of classifiers so far (see \cite{Ji:2014} for a review.) 
Another approach is to use the ensemble as a service to make predictions
for test data followed by appropriate sanitization. 
Suppose we use majority voting to provide a prediction for a test sample. 
A differentially private mechanism such as Report Noisy Max \cite{Dwork:2013:TCS}
can be used to sanitize the votes for a {\it single} query.
However, answering several queries requires perturbing all answers with noise 
linearly proportional to the number of queries, 
which is impractical in most realistic settings.

\subsection{Leveraging auxiliary data}

To address the problems above, we propose to transfer the knowledge of
the ensemble to a global classifier using auxiliary unlabeled data. 
More precisely, we use the ensemble to generate (pseudo)labels for the auxiliary
data, which in turn are used to train a global classifier.
Compared to directly releasing local classifiers, releasing a global classifier
trained on auxiliary data is a much less sensitive operation with $O(M^{-1})$ 
(Sec.~\ref{sec:performance 2}) analogous to releasing an average statistic.
The number of auxiliary samples does not affect privacy, and in fact the
larger the data the closer the global classifier is to the original ensemble 
with $O(N^{-1})$ bound (Sec.~\ref{sec:performance 2}). 
Also, compared to using the ensemble to answer prediction queries, 
the sanitized global classifier can be used as many times as needed without
its privacy being affected. 

We argue that the availability of auxiliary unlabeled data is not an issue, 
since in many settings they are practically much easier to collect than labeled data.
Furthermore, if the auxiliary data are obtained from public repositories,
privacy of such data is not an immediate concern. 
We mainly focus on the privacy of local data,  and discuss extensions 
for preserving the privacy of auxiliary data in Sec.~\ref{sec:extension}.

\section{Finding a global private classifier}\label{sec:step 2}

We present details of training a global private classifier.
As the first attempt, we use majority voting of an ensemble to assign labels
to auxiliary data, and find a global classifier from the usual ERM procedure.  
In the second attempt, we present a better approach where we use the ensemble 
to estimate the posterior $P(y|x)$ of the auxiliary samples 
and solve a `soft-labeled' weighted empirical minimization.

\subsection{First attempt: ERM with majority voting}\label{sec:vote}

As the first attempt, we use majority voting of $M$ local classifiers
to generate labels of auxiliary data, and analyze its implications.
Majority voting for binary classification is the following rule
\begin{equation}\label{eq:voting}
v(x) = \left\{\begin{array}{cc}
 1,&\mathrm{if}\;\;\;\sum_{i=1}^M I[h_i(x)=1] \geq \frac{M}{2}\\
-1,&\mathrm{otherwise}
\end{array}\right. .
\end{equation}
Ties can be ignored by assuming an odd number $M$ of parties.
Regardless of local classifier types or how they are trained, 
we can consider the majority vote of the ensemble $\{h_1,...,h_M\}$
as a {\it deterministic} target concept to train a global classifier.

The majority-voted auxiliary data are
\begin{equation} \label{eq:dataset 1}
S = \{(x_1,v(x_1)),\;...,\;(x_N,v(x_N))\},
\end{equation}
where $x_i \in \mathcal{X}$ is an i.i.d. sample from the same distribution $P(x)$
as the private data.
We train a global classifier by minimizing the (regularized) empirical risk associated
with a loss and a hypothesis class:
\begin{equation}\label{eq:regularized empirical risk}
R^\lambda_S(w) = \frac{1}{N} 
\sum_{(x,v) \in S} l(h(x;w),v) + \frac{\lambda}{2}\|w\|^2.
\end{equation}
The corresponding expected risks with and without regularization are 
\begin{equation}
R^\lambda(w) = E_{x} [l(h(x;w),v(x))] +\frac{\lambda}{2}\|w\|^2,
\end{equation}
and
\begin{equation}
R(w) = E_{x} [l(h(x;w),v(x))].
\end{equation}
Algorithm~\ref{alg:vote} summarizes the procedure.

\begin{algorithm}[tb] \caption{DP Ensemble by Majority-voted ERM} \label{alg:vote}
{Input}: $h_1,...,h_M$ (local classifiers),
 $X$ (auxiliary unlabeled samples), $\epsilon$, $\lambda$\\
{Output}: $w_p$ \\
{Begin}
\begin{algorithmic}
	\FOR{$i=1,\;...\;,N$}
	\STATE{Generate majority voted labels $v(x_i)$ by (\ref{eq:voting})}
	\ENDFOR
    \STATE{Find the minimizer $w_s$ of (\ref{eq:regularized empirical risk})
    with $S=\{(x_i,v(x_i))\}$}
    \STATE{Sample a random vector $\eta$ from
		$p(\eta) \propto e^{-0.5\lambda\epsilon \|\eta\|}$}
    \STATE{Output $w_p = w_s + \eta$}
\end{algorithmic}
\end{algorithm}

Applying output perturbation to our multiparty setting gives us the following result. 
\begin{theorem}\label{thm:privacy 1}
The perturbed output $w_p = w_s + \eta$ from Algorithm~\ref{alg:vote} with
$p(\eta) \propto e^{-\frac{\lambda \epsilon}{2}\|\eta\|}$ 
is $\epsilon$-differentially private.
\end{theorem}
The proof of the theorem and others are in the Appendix~\ref{sec:appendix a}.

\subsection{Performance issues of majority voting}\label{sec:performance issues}

We briefly discuss the generalization error of majority-voted ERM. 
In \cite{Chaudhuri:2011:JMLR}, it is shown that the expected risk of
an output-perturbed ERM solution $w_p$ with respect to the risk of any reference hypothesis 
$w_0$ is bounded by two terms -- one due to noise
and another due to the gap between expected and empirical regularized risks.
This result is applicable to the majority-voted ERM with minor modifications.
The sensitivity of majority-voted ERM from Theorem~\ref{thm:privacy 1} is 
$\frac{2}{\lambda}$ compared to $\frac{2}{N\lambda}$ of a standard ERM,
and corresponding the error bound is
\begin{equation}\label{eq:performance 1}
R(w_p) \leq R(w_0) + O(\epsilon^{-2}) + O(N^{-1}),
\end{equation}
with high probability, ignoring other variables. 
Unfortunately, the bound does not guarantee a successful learning due to the constant
gap $O(\epsilon^{-2})$,
which can be large for a small $\epsilon$.

What causes this is the worst-case scenario of multiparty voting.
Suppose the votes of $M-1$ local classifiers are exactly ties 
for all auxiliary samples $\{x_1,...,x_N\}$. 
If we replace a local classifier $h_i(x)$ with the `opposite' classifier $h_i'(x)=-h_i(x)$, then the majority-voted labels $\{v_1,...,v_N\}$ become
$\{-v_1,...,-v_N\}$, and the resultant global classifier is entirely different.
However unlikely this scenario may be in reality, differential privacy requires that we calibrate 
our noise to the worst case sensitivity.


\subsection{Better yet: weighted ERM with soft labels}\label{sec:soft}

The main problem with majority voting was its sensitivity to the decision of a single party.
Let $\alpha(x)$ be the fraction of positive votes from $M$ classifiers given a
sample $x$: 
\begin{equation}\label{eq:alpha}
\alpha(x) = \frac{1}{M} \sum_{j=1}^M I[h_j(x)=1].
\end{equation}
In terms of $\alpha$, the original loss $l(w^Tx v(x))$ for majority voting
can be written as
\begin{eqnarray}
l(yw^Tx) &=& I[\alpha(x)\geq0.5]\;l(w^Tx) \nonumber \\
&&\;\;+\;\;I[\alpha(x)<0.5]\;l(-w^Tx), 
\end{eqnarray}
which changes abruptly when the fraction $\alpha(x)$ crosses the boundary $\alpha=0.5$.
We remedy the situation by introducing the new {\it weighted} loss:
\begin{equation}\label{eq:weighted loss}
l^\alpha(\cdot) = \alpha(x)l(w^Tx) + (1-\alpha(x))l(-w^Tx).
\end{equation}
The new loss has the following properties.
When the $M$ votes given a sample $x$ are unanimously positive (or negative),
then the weighted loss is $l^\alpha(\cdot) = l(w^Tx)$ (or $l(-w^Tx)$),
same as the original loss.
If the votes are almost evenly split between positive and negative, then 
the weighted loss is
$l^\alpha(\cdot) \simeq 0.5\;l(w^Tx) + 0.5\;l(-w^Tx)$ which is insensitive to 
the change of label by a single vote, unlike the original loss. 
Specifically, a single vote can change $l^\alpha(\cdot)$ only by a factor of $1/M$ 
(see Proof of Theorem~\ref{thm:privacy 2}.)

We provide a natural interpretation of $\alpha(x)$ and the weighted loss in the following.
For the purpose of analysis, assume that the local classifiers 
$h_1(x),...,h_M(x)$ are from the same hypothesis class.\footnote{Our differential privacy guarantee holds whether they are from
the same hypothesis class or not.}
Since the local training data are i.i.d. samples from $P(x,y)$, 
the local classifiers $\{h_1(x),...,h_M(x)\}$ can be considered
random hypotheses, as in \cite{Breiman:1996}.
Let $Q(j|x)$ be the probability of such a random hypothesis $h(x)$ predicting 
label $j$ given $x$:
\begin{equation}
Q(j|x) = P(h(x)=j|x),
\end{equation} 
Then the fraction $\alpha(x)=\frac{1}{M} \sum_{j=1}^M I[h_j(x)=1]$ is an unbiased estimate of $Q(1|x)$. 
Furthermore, the weighted loss is directly related to the unweighted loss:
\begin{lemma} \label{lemma:equivalence}
For any $w$, the expectation of the weighted loss (\ref{eq:weighted loss}) is 
asymptotically the expectation of the unweighted loss: 
\begin{equation}
\lim_{M\to \infty} E_x[l^\alpha(w)] = E_{x,v}[l(w^Tx v)].
\end{equation}
\end{lemma}
\begin{proof}
The expected risk $E_{x,v} [l(vw^Tx)]$ is
\begin{eqnarray}
 &=&  E_x E_{v|x} [l(vw^Tx)] \nonumber\\
&=& E_x[Q(1|x)l(w^T x)+ Q(-1|x) l(-w^T x)]\nonumber \\
&=& E_x[\lim_{M\to\infty} \alpha(x) l(w^T x)+(1-\lim_{M\to\infty} \alpha(x)) l(-w^T x)]\nonumber\\
&& \;\;\;(\mathrm{the\;law\;of\;large\; numbers})\nonumber\\
&=& \lim_{M\to\infty} E_x[\alpha(x)l(w^T x)+ (1-\alpha(x)) l(-w^T x)]\nonumber \\
&& \;\;\;(\mathrm{bounded}\;\alpha\;\mathrm{and}\;l\;\mathrm{for}\;\forall x\in\mathcal{X})\nonumber\\
&=& \lim_{M\to\infty} E_x[l^\alpha(w)].
\end{eqnarray}
\end{proof}
This shows that minimizing the expected weighted loss is asymptotically the same
as minimizing the standard expected loss, when the target $v$ is a
{\it probabilistic} concept from $P(h(x)=v)$ of the random hypothesis,
as opposed to a deterministic concept $v(x)$ from majority voting.

The auxiliary dataset with `soft' labels is now
\begin{equation} \label{eq:dataset 2}
S = \{(x_1,\alpha(x_1)),\;...\;,(x_N,\alpha(x_N))\}.
\end{equation}
where $x_i \in \mathcal{X}$ is an i.i.d. sample from the same distribution $P(x)$
as the private data, and $0 \leq \alpha \leq 1$.
Note that we are not trying to learn a regression function $\mathcal{X} \to [0,1]$ but 
to learn a classifier $\mathcal{X} \to \{-1,1\}$ using $\alpha$ as a {real-valued}
oracle on $P(y=1|x)$. 
Consequently, we find a global classifier by minimizing the regularized {\it weighted} 
empirical risk 
\begin{equation}\label{eq:weighted regularized empirical risk}
R^\lambda_{S}(w)
= \frac{1}{N} \sum_{i=1}^N l^\alpha(h(x_i;w),\alpha_i) + \frac{\lambda}{2}\|w\|^2,
\end{equation}
where $\alpha_i = \alpha(x_i)$. 
The corresponding expected risks with and without regularization are 
\begin{equation}\label{eq:weighted regularized risk}
R^\lambda(w) = E_{x} [l^\alpha(h(x;w),\alpha(x))] +\frac{\lambda}{2}\|w\|^2,
\end{equation}
and
\begin{equation}
R(w) = E_{x} [l^\alpha(h(x;w),\alpha(x))].
\end{equation}

We again use output perturbation to make the classifier differentially
private as summarized in Algorithm~\ref{alg:soft}.

\begin{algorithm}[tb] \caption{DP Ensemble by Weighted ERM} \label{alg:soft}
{Input}: $h_1,...,h_M$ (local classifiers), $X$ (auxiliary unlabeled samples), $\epsilon$, $\lambda$\\
{Output}: $w_p$ \\
{Begin}
\begin{algorithmic}
	\FOR{$i=1,\;...\;,N$}
	\STATE{Compute $\alpha(x_i)$ by (\ref{eq:alpha})}
	\ENDFOR
    \STATE{Find the minimizer of $w_s$ of (\ref{eq:weighted regularized empirical risk}) with $\{(x_i,\alpha(x_i))\}$}
    \STATE{Sample a random vector $\eta$ from
		$p(\eta) \propto e^{-0.5M\lambda\epsilon \|\eta\|}$}
    \STATE{Output $w_p = w_s + \eta$}
\end{algorithmic}
\end{algorithm}

\subsection{Privacy and performance}\label{sec:performance 2}

Compared to Theorem~\ref{thm:privacy 1} for majority-voted ERM with a noise of
$P(\eta) \propto e^{-\frac{\lambda\epsilon}{2}\|\eta\|}$, 
we have the following result:
\begin{theorem}\label{thm:privacy 2}
The perturbed output $w_p = w_s + \eta$ from Algorithm~\ref{alg:soft} with
$p(\eta) \propto e^{-\frac{M \lambda \epsilon}{2}\|\eta\|}$ 
is $\epsilon$-differentially private.
\end{theorem}
That is, we now require $1/M$ times smaller noise to 
achieve the same $\epsilon$-differential privacy.
This directly impacts the performance of the corresponding global classifier 
as follows.
\begin{theorem}\label{thm:performance 2}
Let $w_0$ be any reference hypothesis. Then
with probability of at least $1-\delta_p-\delta_s$
over the privacy mechanism $(\delta_p)$ and over the choice of samples $(\delta_s)$,
\begin{eqnarray}
R(w_p) &\leq& R(w_0) 
+\frac{4d^2(c+\lambda) \log^2(d/\delta_p)}{\lambda^2 M^2 \epsilon^2} \nonumber \\
&& + \frac{16(32+\log(1/\delta_s))}{\lambda N} + \frac{\lambda}{2}\|w_0\|^2. \label{eq:performance 2}
\end{eqnarray}
\end{theorem}

The generalization error bound above has the $O(M^{-2}\epsilon^{-2})$ term 
compared to the $O(\epsilon^{-2})$ term for majority-voted ERM (\ref{eq:performance 1}). 
This implies that by choosing a large $M$, Algorithm~\ref{alg:soft} can find 
a solution whose expected risk is close to the minimum of a non-private solution
for any fixed privacy level $\epsilon>0$. 

We remind the user that the results should be interpreted with a caution. 
The bounds in (\ref{eq:performance 1}) and (\ref{eq:performance 2}) 
indicate the goodness of private ERM solutions relative to the best non-private
solutions with deterministic and probability concepts which are not the same task.
Also, they do not indicate the goodness of the ensemble approach itself 
relative to a centrally-trained classifier using all private data without
privacy consideration.
We leave this comparison to empirical evaluation in the experiment section.


\subsection{Extensions}\label{sec:extension}


We discuss extensions of Algorithms~\ref{alg:vote} and ~\ref{alg:soft}
to provide additional privacy for auxiliary data. 
More precisely, those algorithms can be made $\epsilon$-differentially private
for all private data of a single party and a single sample in the auxiliary data,
by increasing the amount of perturbation as necessary.
We outline the proof as follows. 
In the previous sections, a global classifier was trained on 
auxiliary data whose labels were generated either by majority voting 
or soft labeling. A change in the local classifier affects only the labels $\{v_i\}$
of the auxiliary data but not the features $\{x_i\}$. 
Now assume in addition that the feature of one sample from the auxiliary data
can also change arbitrarily, {i.e.,} $x_j \ne x'_j$ for some $j$ and $x_i=x'_i$ for all
$i \in \{1,...,N\}\setminus\{j\}$. The sensitivity of the resultant risk minimizer
can be computed similarly to the proofs of Theorems~\ref{thm:privacy 1} and
\ref{thm:privacy 2} in Appendix~\ref{sec:appendix a}.
Briefly, the sensitivity is upper-bounded by the absolute sum of the difference
of gradients 
\begin{equation} \label{eq:nabla g}
\|\nabla g(w)\| \leq \frac{1}{N}\sum_{i=1}^N \|\nabla l(y_iw^Tx_i) - \nabla l(y_i'w^Tx_i')\|.
\end{equation}
For majority voting, one term in the sum (\ref{eq:nabla g}) is
\begin{equation}
\|v(x_j)x_j l'(v(x_j)w^Tx_j) - v'(x'_j) x'_j l'(v'(x'_j) w^T x'_j)\|
\end{equation}
which is at most 2 for any $x_j,x_j' \in \mathcal{X}$, 
and therefore the sensitivity is the same whether $x_j = x'_j$ or not. 
As a result, Algorithm~\ref{alg:vote} is already $\epsilon$-differentially private
for both labeled and auxiliary data without modification. 
Furthermore, the privacy guarantee remains the same if we allow $x_j\ne x'_j$ for
any number of samples. 
For soft labeling, one term in the sum (\ref{eq:nabla g}) is 
\begin{eqnarray}
&&\|\alpha_j x_j l'(w^Tx_j) -(1-\alpha_j) x_j l'(-w^Tx_j) \nonumber\\
&& -\alpha'_j x_j' l'(w^Tx'_j) +(1-\alpha'_j) x'_j l'(-w^Tx'_j)\|
\end{eqnarray}
which is also at most 2 for any $x_j,x_j' \in \mathcal{X}$ and $\frac{2}{M}$
when $x_j=x_j'$. 
When only a single auxiliary sample changes, i.e., $x_j\ne x_j'$ for one $j$, 
the overall sensitivity increases by a factor of $\frac{N+M-1}{N}$. 
By increasing the noise by this factor,
Algorithm~\ref{alg:soft} is $\epsilon$-differentially private
for both labeled and auxiliary data.
Note that this factor $\frac{N+M-1}{N}$ can be bounded close to 1 if we increase
the number of auxiliary samples $N$ relative to the number of parties $M$. 

\section{Related work} \label{sec:related work}

To preserve privacy in data publishing, several approaches such as 
$k$-anonymity \cite{Sweeney:2002:IJUFKS} and secure multiparty computation \cite{Yao:1982:FOCS}
have been proposed (see \cite{Fung:2010:CSUR} for a review.)
Recently, differential privacy \cite{Dwork:2004:CRYPTO,Dwork:2006:TC,Dwork:2006:ALP}
has addressed several weaknesses of $k$-anonymity \cite{Ganta:2008:SIGKDD}, 
and gained popularity as a quantifiable measure of privacy risk. 
The measure provides a bound on the privacy loss regardless of 
any additional information an adversary might have.
Differential privacy has been used for a privacy-preserving data analysis platform \cite{Mcsherry:2009:SIGMOD}, and
for sanitization of learned model parameters from a standard ERM 
\cite{Chaudhuri:2011:JMLR}.
This paper adopts output perturbation techniques from the latter
to sanitize non-standard ERM solutions from multiparty settings.

Private learning from multiparty data has been studied previously.
In particular, several differentially-private algorithms were proposed, including
parameter averaging through secure multiparty computation \cite{Pathak:2010}, 
and private exchange of gradient information to minimize empirical risks incrementally
\cite{Rajkumar:2012,Hamm:2015}.
Our paper is motivated by \cite{Pathak:2010} but uses a very different approach
to aggregate local classifiers. In particular, we use an ensemble approach and
average the classifier decisions \cite{Breiman:1996} instead of parameters,
which makes our approach applicable to arbitrary and mixed classifier types.
Advantages of ensemble approaches in general have been analyzed previously,
in terms of bias-variance decomposition \cite{Breiman:1996b}, 
and in terms of the margin of training samples \cite{Schapire:1998}. 

Furthermore, we are using unlabeled data to augment labeled data during training,
 which can be considered a semi-supervised learning method \cite{Chapelle:2006}. 
There are several related papers in this direction.
Augmenting private data with non-private {\it labeled} data to lower the sensitivity
of the output is straightforward, and was demonstrated in medical applications \cite{Ji:2014}.
Using non-private {\it unlabeled} data, which is more general than using labeled data, 
was demonstrated specifically to assist learning of random forests \cite{Jagannathan:2013}.
Our use of auxiliary data is not specific to classifier types. 
Furthermore, we present an extension to preserve the privacy of auxiliary data as well.

\section{Experiments} \label{sec:experiments}

We use three real-world datasets to compare the performance of
the following algorithms:
\begin{itemize}\setlength{\itemsep}{-2pt}
\item {\it batch}: classifier trained using all data ignoring privacy
\item {\it soft}: private ensemble using soft-labels (Algorithm~\ref{alg:soft})
\item {\it avg}: parameter averaging \cite{Pathak:2010}
\item {\it vote}: private ensemble using majority voting (Algorithm~\ref{alg:vote})
\item {\it indiv}: individually trained classifier using local data
\end{itemize}
We can expect {\it batch} to perform better than any private algorithm 
since it uses all private data for training ignoring privacy.
In contrast, {\it indiv} uses only the local data for training and will 
perform significantly worse than {\it batch}, but it achieves a perfect privacy 
as long as the trained classifiers are kept local to the parties. 
We are interested in the range of $\epsilon$ where private algorithms ({\it soft},
{\it avg}, and {\it vote}) perform better than the baseline {\it indiv}. 

To compare all algorithms fairly, we use only a single type of
classier -- binary or multiclass logistic regression. For Algorithms~\ref{alg:vote}
and \ref{alg:soft}, both local and global classifiers are of this type as well.
The only hyperparameter of the model is the regularization coefficient
$\lambda$ which we fixed to $10^{-4}$ after performing some preliminary experiments.
About $10\%$ of the original training data are used as auxiliary unlabeled data,
and the rest $90\%$ are randomly distributed to $M$ parties as private data. 
We report the mean and s.d.~over 10 trials for non-private
algorithms and 100-trials for private algorithms. 

\begin{figure}[htb]
\centering
\includegraphics[width=0.8\linewidth]{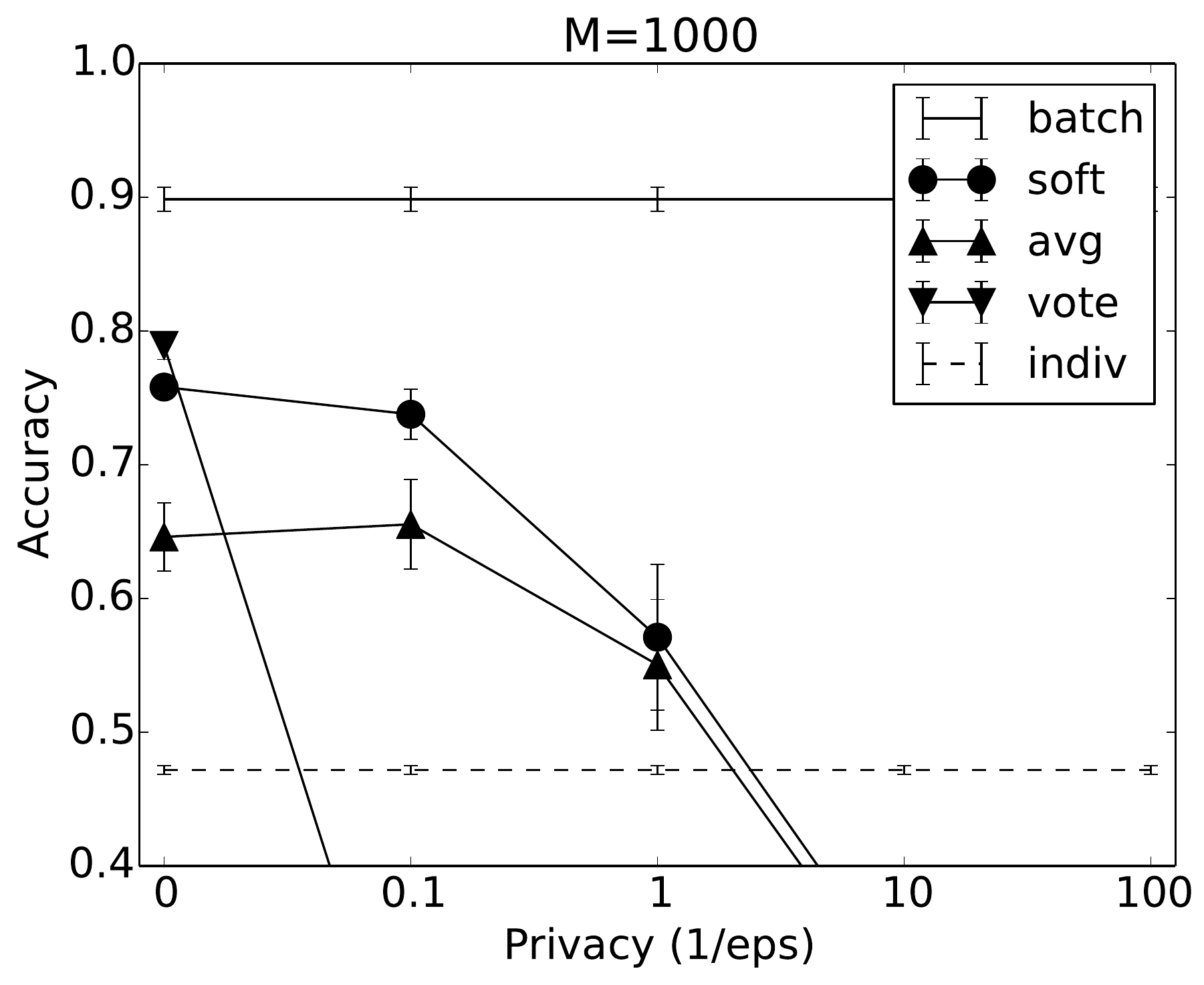}
\caption{
Test accuracy of private and non-private algorithms for activity recognition
with $K=6$ activity types. 
}
\label{fig:activity recognition}
\end{figure}

\begin{figure*}[htb]
\centering
\includegraphics[width=0.33\linewidth]{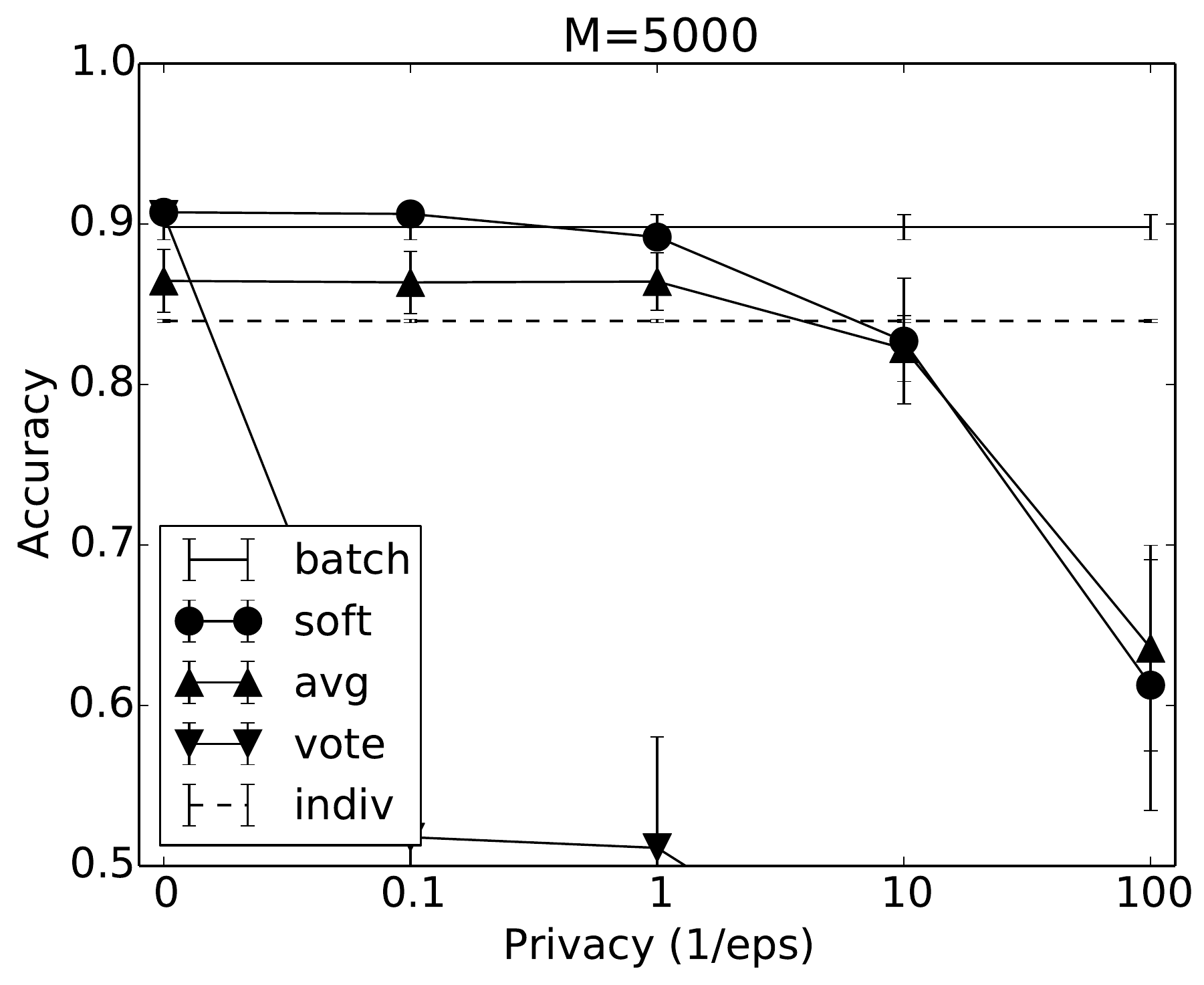}
\includegraphics[width=0.33\linewidth]{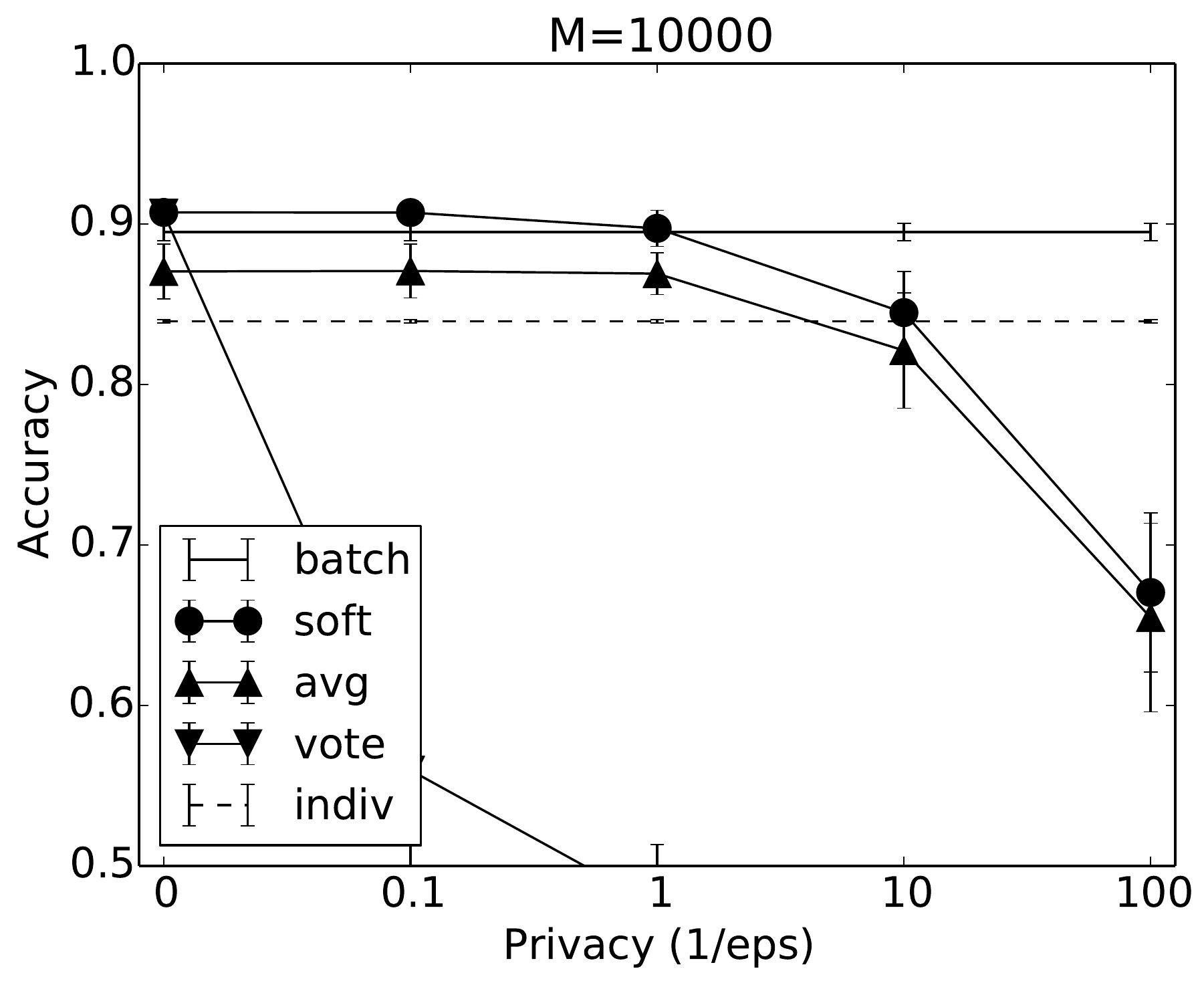}
\includegraphics[width=0.33\linewidth]{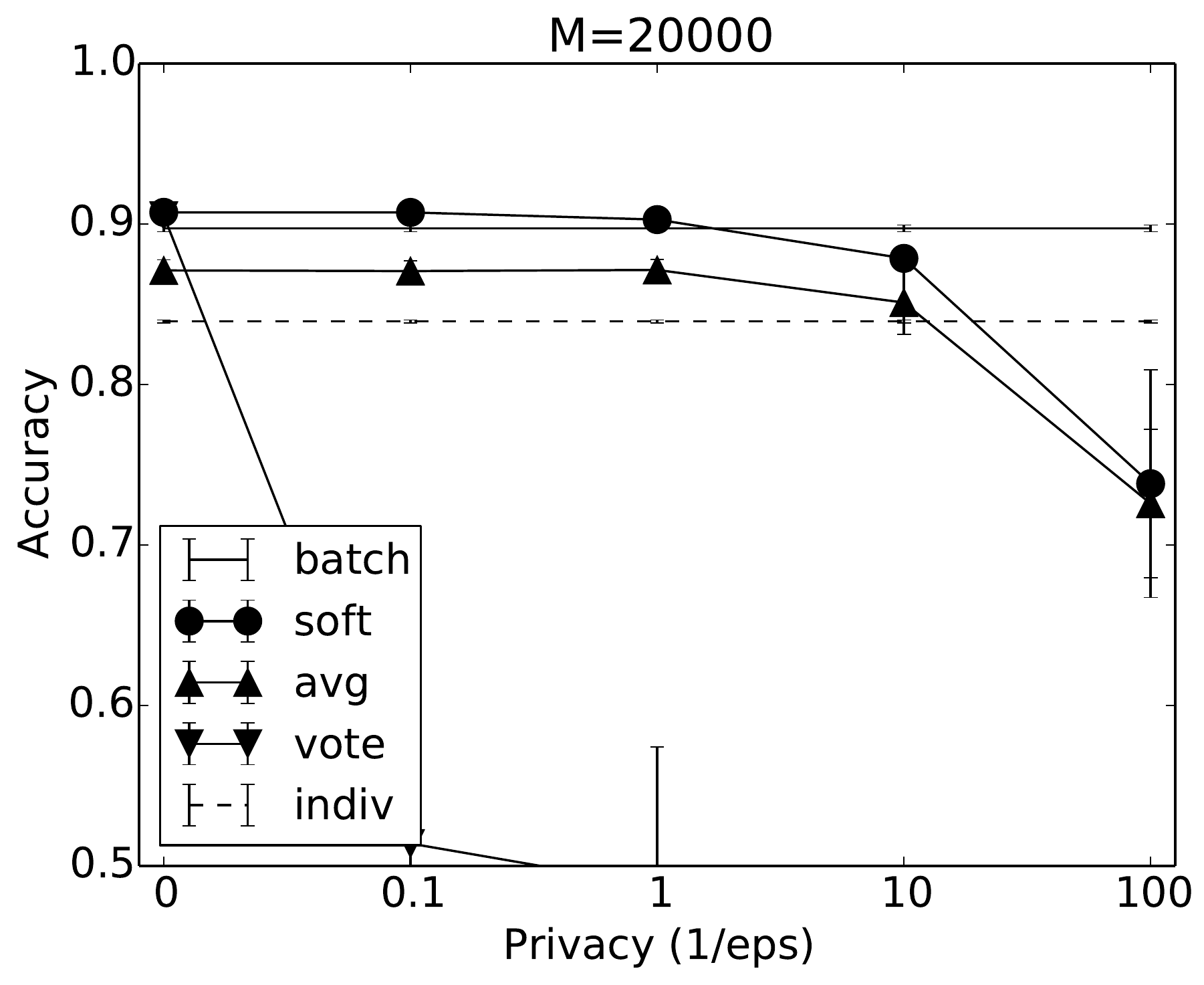}
\caption{
Test accuracy of private and non-private algorithms for network intrusion detection.
}
\label{fig:network intrusion detection}
\end{figure*}

\begin{figure*}[htb]
\centering
\includegraphics[width=0.33\linewidth]{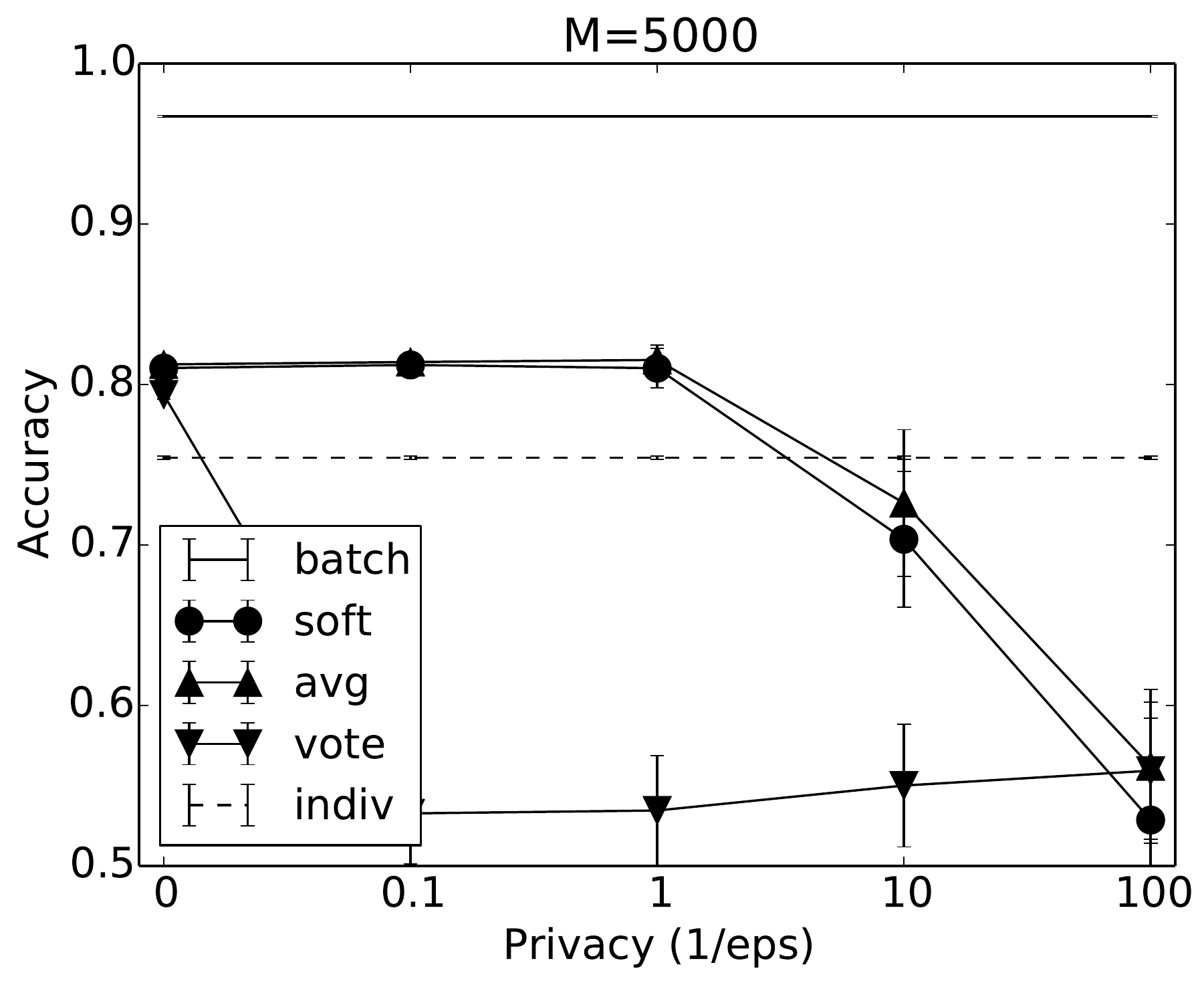}
\includegraphics[width=0.33\linewidth]{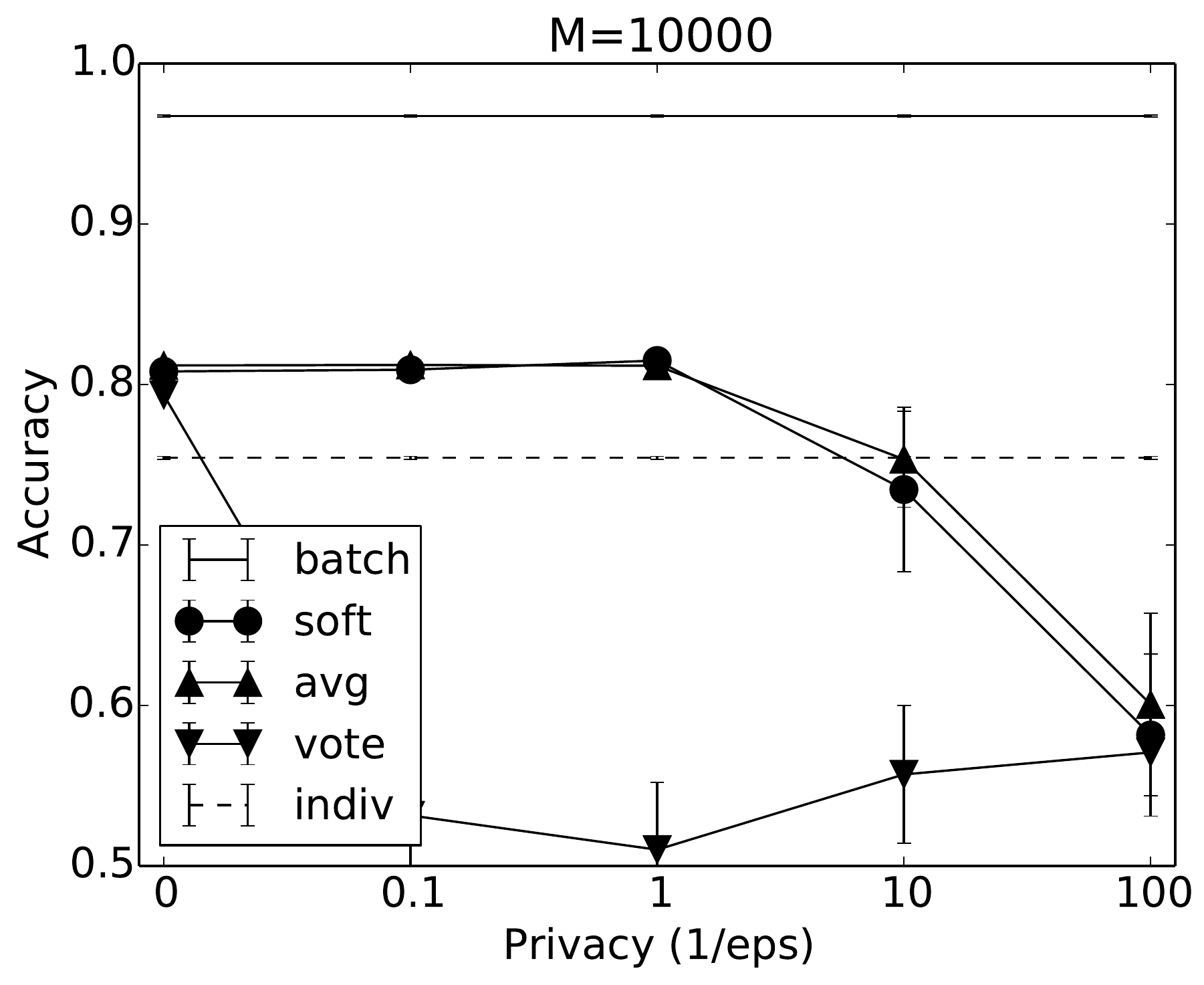}
\includegraphics[width=0.33\linewidth]{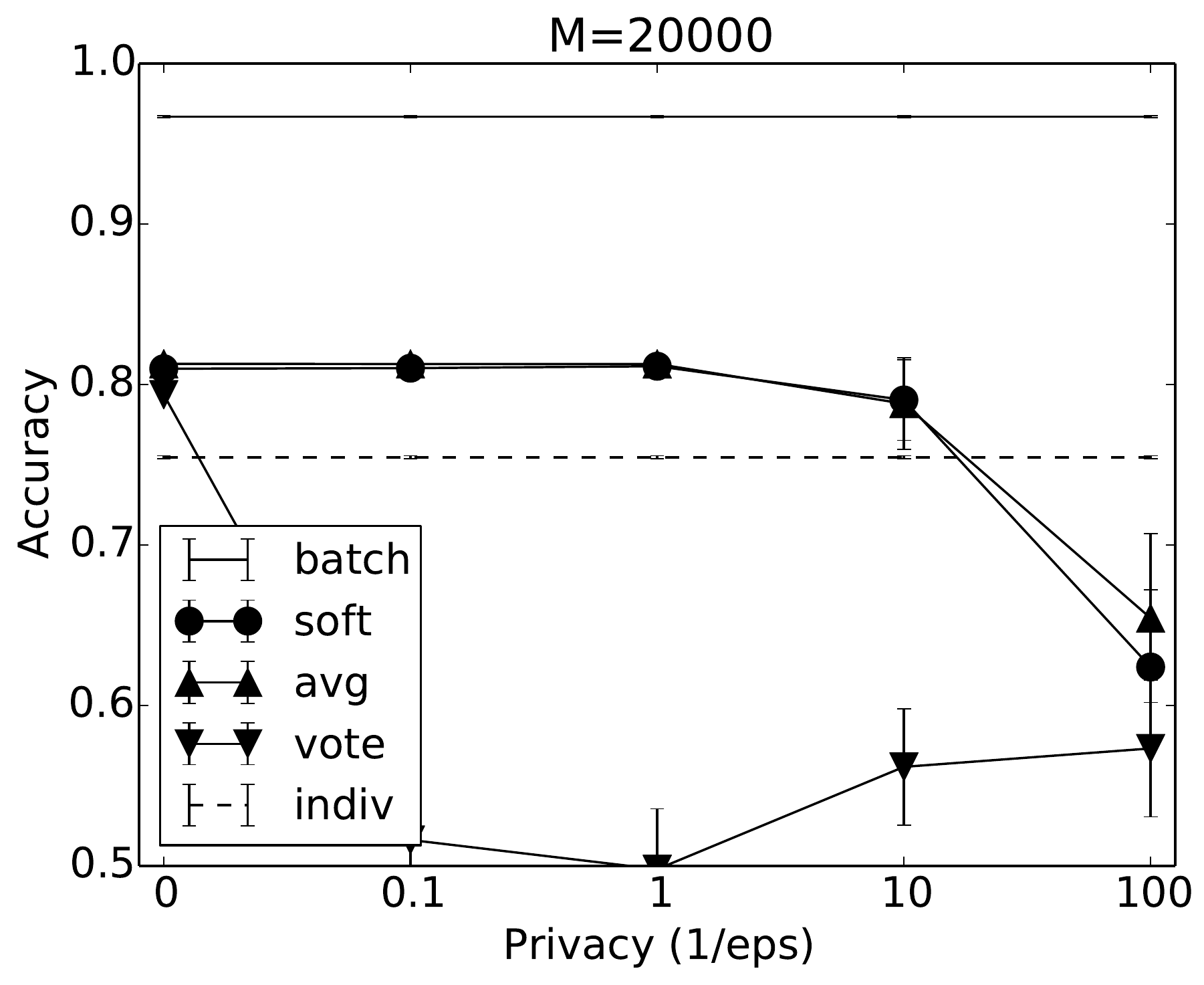}
\caption{
Test accuracy of private and non-private algorithms for malicious URL detection.
}
\label{fig:malicious URL detection}
\end{figure*}

\subsection{Activity recognition using accelerometer}

Consider a scenario where wearable device users want to train a motion-based
activity classifier without revealing her data to others.
To test the algorithms, we use the UCI Human Activity Recognition Dataset \cite{Anguita:2012:IWAAL},
which is a collection of motion sensor data on a smart device by multiple subjects
performing $6$ activities ({\it walking}, {\it walking upstairs}, 
{\it walking downstairs}, {\it sitting}, {\it standing}, {\it laying}).
Various time and frequency domain variables are extracted from the signal,
and we apply PCA to get $d=50$ dimensional features.
The training and testing samples are $7K$ and $3K$, respectively.

We simulate a case with $M=1K$ users (i.e.,~parties). Each user can use only
$6$ samples to train a local classifier. 
The remaining $1K$ samples are used as auxiliary unlabeled data. 
Figure~\ref{fig:activity recognition} shows the test accuracy of using different
algorithms with varying privacy levels. 
For non-private algorithms, 
the top solid line ({\it batch}) shows the accuracy of a batch-trained classifier
at around $0.90$, and the bottom dashed line ({\it indiv}) shows the averaged accuracy
of local classifiers at around $0.47$.
At $1/\epsilon = 0$, the private algorithms achieve test accuracy of
0.79 ({\it vote}), 0.76 ({\it soft}) and 0.67 ({\it avg}), and as the the 
privacy level $1/\epsilon$ increases, the performance drops for all private algorithms. 
As expected from the bound (\ref{eq:performance 1}), 
{\it vote} becomes useless even at  $1/\epsilon=0.1$, while 
{\it soft} and {\it avg} are better than {\it indiv} until $1/\epsilon=1$. 
We fixed $M$ to 1000 in this experiment due to the limited number of samples, 
the tendency in the graph is similar to other experiments with larger $M$'s.

\subsection{Network intrusion detection}

Consider a scenario where multiple gateways or routers collect suspicious network
activities independently,
and aim to collaboratively build an accurate network intrusion detector without
revealing local traffic data.
For this task we use the KDD-99 dataset, which consists of examples of `bad'
connections, called intrusions or attacks, and `good' normal connections.
Features of this dataset consists of continuous values and categorical attributes. 
To apply logistic regression, we change categorical attributes to one-hot vectors 
to get $d=123$ dimensional features. 
The training and testing samples are $493K$ and $311K$, respectively.

We simulate cases with $M=5K/10K/20K$ parties. Each party can only use $22$ samples
to train a local classifier.
The remaining $43K$ samples are used as auxiliary unlabeled data. 
 Figure~\ref{fig:network intrusion detection} shows the test accuracy of using different
algorithms with varying privacy levels. 
For each of $M=5K/10K/20K$ the tendency of algorithms is similar to
Figure~\ref{fig:activity recognition} -- for a small $1/\epsilon$, private algorithms
perform roughly in between {\it batch} and {\it indiv}, and as 
$1/\epsilon$ increases
private algorithms start to perform worse than {\it indiv}. 
When $M$ is large (e.g., $M=20K$), private algorithms {\it soft} and {\it avg}
hold their accuracy better than when $M$ is small (e.g., $M=5K$.) 
In particular, {\it soft} performs nearly as well as the non-private {\it batch} 
until around $1/\epsilon=10$ compared to ${\it avg}$.

\subsection{Malicious URL prediction}

In addition to network intrusion detection, multiple parties
such as computer users can collaborate on detecting malicious URLs without
revealing the visited URLs.
The Malicious URL Dataset \cite{Ma:2009} is a collection of examples
of malicious URLs from a large Web mail provider. 
The task is to predict whether a URL is malicious or not by 
various lexical and host-based features of the URL.
We apply PCA to get $d=50$ dimensional feature vectors.
We choose days 0 to 9 for training, and days 10 to 19 for testing, which
amount to $200K$ samples for training and $200K$ samples for testing.

We simulate cases with $M=5K/10K/20K$ parties. Each party can use only $9$ samples
to train a local classifier.
The remaining $16K$ samples are used as auxiliary unlabeled data. 
Figure~\ref{fig:malicious URL detection} shows the test accuracy of using different
algorithms with varying privacy levels. 
The gap between {\it batch} and other algorithms is larger compared to the previous 
experiment (Figure~\ref{fig:network intrusion detection}), most likely due to 
the smaller number (=9) of samples per party. 
However, the overall tendency is very similar to previous experiments.

\section{Conclusion} \label{sec:conclusion}

In this paper, we propose a method of building global differentially private 
classifiers from local classifiers using two new ideas: 
1) leveraging unlabeled auxiliary data to transfer the knowledge of the 
ensemble, and 2) solving a weighted ERM using class probability estimates
from the ensemble.
In general, privacy comes with a loss of classification performance.
We present a solution to minimize the performance gap between private and
non-private ensembles demonstrated with real world tasks. 


{\small
\bibliographystyle{icml2016}
\bibliography{icml16_1_jh,nips14_1_jh,icdcs15_jh}
}

\appendix

\section{Proofs} \label{sec:appendix a}

\subsection{Proof of Theorem~\ref{thm:privacy 1}}
{\it
Theorem~\ref{thm:privacy 1}: 
The perturbed output $w_p = w_s + \eta$ from Algorithm~\ref{alg:vote} with
$p(\eta) \propto e^{-\frac{\lambda \epsilon}{2}\|\eta\|}$ 
is $\epsilon$-differentially private.
}
\begin{proof}
We will compute the sensitivity of the minimizer $w_s$ of
the regularized empirical risk with majority-voted labels (\ref{eq:regularized empirical risk}). 
Suppose $\mathcal{D} =(S^{(1)},...,S^{(M)})$ is the ordered set of 
private training data (\ref{eq:private data}) for $M$ parties, 
and $\mathcal{D}' =((S')^{(1)},...,S^{(M)})$ is a neighboring set 
which differs from $\mathcal{D}$ only at party 1's data, without loss of generality.
The local classifiers after training with $\mathcal{D}$ and $\mathcal{D}'$ are
$H = (h_1,...,h_M)$ and $H' = (h_1', ..., h_M)$, respectively, which are 
again different only for classifier 1.
The majority votes $v(x)$ and $v'(x)$ from $\mathcal{D}$ and $\mathcal{D}'$
generates two auxiliary training sets
$S = \{(x_i,v(x_i))\}$ and $S' = \{(x_i,v'(x_i)\}$ which have the same features
but possibly different labels. 

Let $R^\lambda_S(w)$ and $R^\lambda_{S'}(w)$ be the regularized empirical risks for
training sets $S$ and $S'$, and let $w_s$ and $w_{s'}$ be the minimizers of the
respective risks. 
From Corollaries 7 and 8 \cite{Chaudhuri:2011:JMLR}, the $L_2$ difference of
$w_s$ and $w_{s'}$ is bounded by
\begin{equation}
\|w_s-w_{s'}\| \leq \frac{1}{\lambda}\max_w \|\nabla g(w)\|,
\end{equation}
where $g(w)$ is the risk difference $R^\lambda_S(w) - R^\lambda_{S'}(w)$,
which, in our case, satisfies
\begin{eqnarray}\label{eq:risk diff 1}
\|\nabla g(w)\| &\leq& \frac{1}{N}
\sum_{i=1}^N \|v(x_i)x_i l'(v(x_i)w^Tx_i) \nonumber\\
&&\;\; - v'(x_i) x_i l'(v'(x_i) w^T x_i)\|. \nonumber\\
&\leq& \frac{1}{N}\sum_{i=1}^N \|x_i\|\times\nonumber\\
&& |l'(w^Tx_i)+l'(-w^Tx_i)|.
\end{eqnarray}
Recall that $\|x\|\leq 1$ and $|l'(\cdot)|\leq 1$ by assumption.
In the worst case, $v(x_i) \neq v'(x_i)$ for all $i=1,...,N$, 
and therefore the RHS of (\ref{eq:risk diff 1}) is bounded by 2.
Consequently, the $L_2$ sensitivity of the minimizer $w_s$ is 
\begin{equation}
\max_{S,S'} \|w_s-w_{s'}\| \leq \frac{2}{\lambda}.
\end{equation}
$\epsilon$-differential privacy follows from the sensitivity
result (\ref{eq:sensitivity method}). 
\end{proof}

\subsection{Proof of Theorem~\ref{thm:privacy 2}}
{\it
Theorem~\ref{thm:privacy 2}: 
The perturbed output $w_p = w_s + \eta$ from Algorithm~\ref{alg:soft} with
$p(\eta) \propto e^{-\frac{M \lambda \epsilon}{2}\|\eta\|}$ 
is $\epsilon$-differentially private.
}

\begin{proof}
The proof parallels the proof of Theorem~\ref{thm:privacy 1}.
We again assume $\mathcal{D} =(S^{(1)},...,S^{(M)})$ is the ordered set of 
private training data (\ref{eq:private data}) for $M$ parties, 
and $\mathcal{D}' =((S')^{(1)},...,S^{(M)})$ is a neighboring set 
which differs from $\mathcal{D}$ only at party 1's data, without loss of generality.
Let $S = \{(x_i,\alpha_i)\}$ and $S' = \{(x_i,\alpha'_i)\}$ be the two
resulting datasets which have the same the features but possibly different
$\alpha$'s.
We first compute the sensitivity of the minimizer of
the weighted regularized empirical risk (\ref{eq:weighted regularized empirical risk}). 
Let $R^\lambda_S(w)$ and $R^\lambda_{S'}(w)$ be the regularized empirical risks for
training sets $S$ and $S'$, and let $w_s$ and $w_{s'}$ be the minimizers of 
the respective risks. 
Also let $g(w)$ be the difference $R^\lambda_{S}(w) - R^\lambda_{S'}(w)$ of two risks
\begin{eqnarray}
g(w)&=& \frac{1}{N} \sum_{i=1}^N[\alpha_i l(w^Tx_i)+(1-\alpha_i)l(-w^Tx_i) \nonumber\\
&& - \alpha'_i l(w^Tx_i)-(1-\alpha'_i)l(-w^Tx_i)].
\end{eqnarray}
The gradient of $g(w)$ is bounded by 
\begin{eqnarray}
\|\nabla g(w)\| &\leq& 
\frac{1}{N} \sum_{i=1}^N [ |\alpha_i-\alpha'_i| \|x_i\| |l'(w^Tx_i)| \nonumber\\
&& + |\alpha_i-\alpha'_i|\|x_i\| |l'(-w^Tx_i)|]\label{eq:risk diff 2}\\
&\leq& \frac{1}{N} \sum_{i=1}^N 2|\alpha_i-\alpha'_i|.
\end{eqnarray}
In the worst case,  $\alpha_i \ne \alpha_i'$ for all $i=1,...,N$. 
Since $\alpha_i$ is the fraction of positive votes, $|\alpha_i-\alpha_i'| \leq 1/M$
holds for all $i=1,...,N$.
Therefore the $L_2$ sensitivity of the minimizer $w_s$
is at most $\frac{2}{\lambda M}$ and  
the $\epsilon$-differential privacy follows.
\end{proof}

\subsection{Lemma~\ref{lemma:gamma}}

We use the following lemma. 
\begin{lemma}[Lemma 17 of \cite{Chaudhuri:2011:JMLR}]\label{lemma:gamma}
If $X \sim \Gamma(k,\theta)$, where $k$ is an integer, then
with probability of at least $1-\delta$,
\[
X \leq k\theta \log (k/\delta).
\]
\end{lemma}

\subsection{Lemma~\ref{lemma:noise}}

\begin{lemma}\label{lemma:noise}
If $w_s$ is the minimizer of (\ref{eq:weighted regularized empirical risk}) and 
$w_p$ is the $\epsilon$-differentially private version from Algorithm~\ref{alg:soft},
then with probability of at least $1-\delta_p$ over the privacy mechanism,
\begin{equation}
R^{\lambda}_S(w_p) \leq R^{\lambda}_S(w_s) + \frac{2d^2(c+\lambda)\log^2(d/\delta)}{\lambda^2 M^2 \epsilon^2}
\end{equation}
\end{lemma}

\begin{proof}
A differentiable function $f:\mathbb{R}^d \to \mathbb{R}$ is called
$\beta$-smooth, if $\exists \beta >0$ such that $\|\nabla f(v)-\nabla f(u)\|
\leq \beta \|v-u\|$ for all $u,v$.
From the Mean Value Theorem, such a function satisfies
\[
f(v) \leq f(u) + \nabla^T f(u)(v-u) + \frac{\beta}{2}\|v-u\|^2,\;\forall u,v.
\]

Since $|l'(\cdot)|$ is $c$-Lipschitz, $R^\lambda_S(w)$ is $(c+\lambda)$-smooth:
\begin{eqnarray}
&& \|\nabla R^\lambda_S(v)- \nabla R^\lambda_S(u)\| \nonumber \\
&&\leq  \frac{1}{N} \sum_i \left \| \alpha_i x_i l'(v^Tx_i) - (1-\alpha_i)x_i l'(-v^Tx_i) \right. \nonumber\\
&& \left. - \alpha_i x_i l'(u^Tx_i) + (1-\alpha_i)x_i l'(-u^T x_i)\right\| \nonumber\\
&& +\lambda\|v-u\| \nonumber \\
&& \leq \frac{1}{N}\sum_i \left[\alpha_i c \|(v-u)^Tx_i\| + \right.\nonumber \\
&&  \left.(1-\alpha_i)c \|(u-v)^Tx_i\|\right] + \lambda\|v-u\| \nonumber\\
&& \leq (c+\lambda)\|u-v\|. 
\end{eqnarray}
By setting $v = w_p$ and $u=w_s$ and using the $(c+\lambda)$-smoothness of
$R^\lambda_S(w)$, we have 
\begin{eqnarray}
R^{\lambda}_S(w_p) &\leq& R^{\lambda}_S(w_s) + \nabla^T R^\lambda_S(w_s) (w_p-w_s)\nonumber\\
&& \;\;\;+ \frac{(c+\lambda)}{2} \|w_p-w_s^\ast\|^2 \nonumber\\
&=& R^{\lambda}_S(w_s) + \frac{(c+\lambda)}{2} \|w_p-w_s\|^2.
\end{eqnarray}

Since 
\begin{equation}
P\left(\|w_p-w_s^\ast\|\leq \frac{2d\log(d/\delta)}{\lambda M \epsilon}\right) \geq 1-\delta_p
\end{equation}
from Lemma~\ref{lemma:gamma} with $k=d$ and $\theta = \frac{2}{\lambda M \epsilon}$,
we have the desired result.
\end{proof}

\subsection{Proof of Theorem~\ref{thm:performance 2}}

{\it
Theorem~\ref{thm:performance 2}: 
Let $w_0$ be any reference hypothesis. Then
with probability of at least $1-\delta_p-\delta_s$
over the privacy mechanism $(\delta_p)$ and over the choice of samples $(\delta_s)$,
\begin{eqnarray}
R(w_p) &\leq& R(w_0) 
+\frac{4d^2(c+\lambda) \log^2(d/\delta_p)}{\lambda^2 M^2 \epsilon^2} \nonumber \\
&& + \frac{16(32+\log(1/\delta_s))}{\lambda N} + \frac{\lambda}{2}\|w_0\|^2. 
\end{eqnarray}
}
\begin{proof}
Let $w_s$ and $w^\ast$ be the minimizers of the regularized empirical risk
(\ref{eq:weighted regularized empirical risk}) and the regularized expected risk
(\ref{eq:weighted regularized risk}), respectively. 
The risk at $w_p$ relative to a reference classifier $w_0$ can be written as
\begin{eqnarray}
R(w_p) - R(w_0) &=& R^{\lambda}(w_p) - R^{\lambda}(w^\ast) \nonumber\\
&& + R^{\lambda}(w^\ast) - R^{\lambda}(w_0) \nonumber \\
&& + \frac{\lambda}{2}\|w_0\|^2 -\frac{\lambda}{2}\|w_p\|^2 \nonumber\\
&\leq & R^{\lambda}(w_p) - R^{\lambda}(w^\ast) \nonumber +\frac{\lambda}{2}\|w_0\|^2.\\
\label{eq:thm5 ineq}
\end{eqnarray}
The inequality above follows from $R^{\lambda}(w^\ast) \leq R^{\lambda}(w_0)$ by definition. 
Note that since $\|x\|\leq 1$ and $|l'|\leq 1$ by assumption,
the weighted loss $\alpha(x)l(w^Tx) + (1-\alpha(x))l(w^Tx)$
is $1$-Lipschitz in $w$. 
From Theorem 1 of \cite{Sridharan:2009} with $a=1$,  
we can also bound $R^{\lambda}(w_p) - R^{\lambda}(w^\ast)$ as
\begin{eqnarray}
R^{\lambda}(w_p) - R^{\lambda}(w^\ast) &\leq& 2(R^{\lambda}_S(w_p) - R^{\lambda}_S(w_s^\ast))\nonumber \\
&& + \frac{16(32+\log(1/\delta_s))}{\lambda N}
\end{eqnarray}
with probability of $1-\delta_s$ over the choice of samples.
By combining this inequality with Lemma~\ref{lemma:noise}
using the union bound, we have
\begin{eqnarray}
R^{\lambda}(w_p) - R^{\lambda}(w^\ast) &\leq& \frac{4d^2(c+\lambda) \log^2(d/\delta_p)}{\lambda^2 M^2 \epsilon^2} \nonumber\\
&& + \frac{16(32+\log(1/\delta_s))}{\lambda N}.
\end{eqnarray}
The theorem follows from (\ref{eq:thm5 ineq}).

\end{proof}

\section{Differentially private multiclass logistic regression}\label{sec:appendix b}

We extend our methods to multiclass classification problems and
provide a sketch of $\epsilon$-differential privacy proofs for
multiclass logistic regression loss.

\subsection{Standard ERM} \label{sec:standard erm}
Suppose $y \in 1,...,K$, and let $w = [w_1;\;...\;; w_K]$ be a stacked 
$(d\;K) \times 1$ vector. 
The multiclass logistic loss ({i.e.} softmax) is
\begin{equation}
l(h(x),y) = -w_{y}^Tx + \log(\sum_l e^{w_l^Tx}),
\end{equation}
and the regularized empirical risk is 
\begin{equation}\label{eq:multiclass}
R^\lambda_S(w) = 
-\frac{1}{N}\sum_i  [w_{y_i}^Tx_i - \log(\sum_l e^{w_l^Tx_i})]+ \frac{\lambda}{2} \|w\|^2.
\end{equation}
Note that $R^\lambda_S(w)$ is $\lambda$-strongly convex in $w$.

The sensitivity of $w_s$ which minimizes (\ref{eq:multiclass}) can 
be computed as follows.
Suppose $S$ and $S'$ are two different datasets which are not necessarily 
neighbors: 
$S = \{(x_i,y_i))\}$ and $S' = \{(x'_i,y'_i)\}$.
Let $g(w)$ be the difference $R^\lambda_S(w) - R^\lambda_{S'}(w)$ of the
two risks. 
Then the partial gradient w.r.t. $w_k$ is
\begin{equation}
\nabla_{w_k} R^\lambda_S(w) = -\frac{1}{N}\sum_i x_i \Delta_k(x_i,y_i,w)  + \lambda w_k,
\end{equation}
where
\begin{equation}
\Delta_k(x_i,y_i,w) = I[y_i=k] - \frac{e^{w_k^Tx_i}}{\sum_l e^{w_l^Tx_i}}=
 I[y_i=k] - P_k(x_i).
\end{equation}
Since $I[y_i=k]$ can be non-zero (i.e. 1) for only one $k$, and
$\sum_k P_k(x_i) = 1$ with $0\leq P_k(x_i)\leq 1$, we have
\begin{equation}
\sum_k \Delta_k^2 = \sum_k (I_k - P_k)^2 \leq \sum_k (I_k^2 + P_k^2) \leq 2,
\end{equation}
Let $\Delta(x_i,y_i,w)=[\Delta_1(x_i,y_i,w),...,\Delta_K(x_i,y_i,w)]$ be a 
$K \times 1$ vector (which depends on $x_i,y_i,w$.) 
The gradient of the risk difference $g(w)$ is then
\begin{equation}
\nabla g(w) = 
-\frac{1}{N}\sum_i \Delta(x_i,y_i,w)\otimes x_i - \Delta(x_i',y_i',w)\otimes x_i',
\end{equation}
where $\otimes$ is a Kronecker product of two vectors.
Note that 
\begin{equation}
\|\Delta \otimes x\|^2 = \sum_k \|\Delta_k x\|^2 \leq \|x\|^2 \sum_k \Delta_k^2
\leq 2 \|x\|^2. 
\end{equation}

Without loss of generality, we assume that only $(x_1,y_1)$ and $(x'_1,y'_1)$ are
possibly different and $(x_i,y_i)=(x_i',y_i')$ for all $i=2,...,N$.
In this case we have
\begin{eqnarray}
\| \nabla g(w)\| & \leq & \frac{1}{N} \| \Delta(x_1,y_1,w)\otimes x_1\|\nonumber\\
&& + \frac{1}{N} \| \Delta(x'_1,y'_1,w)\otimes x'_1\| \nonumber\\
&\leq & \frac{\sqrt{2}}{N}(\|x_1\|+\|x'_1\|) \leq \frac{2\sqrt{2}}{N},
\end{eqnarray}
and the therefore the $L_2$ sensitive of the minimizer of a multiclass logistic
regression is
\begin{equation}
\frac{2\sqrt{2}}{N \lambda } 
\end{equation}
from Corollaries 7 and 8 \cite{Chaudhuri:2011:JMLR}.
Note that the sensitivity does not depend on the number of classes $K$.

\subsection{Majority-voted ERM}

Let $S=\{(x_i,v_i)\}$ and $S'=\{(x_i,v_i')\}$ be two datasets with the same
features but with possibly different labels for all $i=1,...,N$.
Then the partial gradient of the risk difference $g(w)$ is
\begin{eqnarray}
\nabla_{w_k} g(w) &=& -\frac{1}{N}\sum_i x_i[I[v_i=k]-I[v'_i=k]] \nonumber\\
&=& -\frac{1}{N}\sum_i x_i a_k(v_i,v_i'),
\end{eqnarray}
where $a_k(v_i,v_i')$ is
\begin{equation}
a_k(v_i,v'_i) = I[v_i=k]-I[v'_i=k] \in \{-1,0,1\}.
\end{equation}
Let $a = [a_1,...,a_K]$ be a $K \times 1$ vector (which depends on $v_i,v_i'$.) 
Note that at most two elements of $a$ can be nonzero (i.e.~$\pm 1$.)
The gradient can be rewritten using the Kronecker product $\otimes$ as
\begin{equation}
\nabla g(w) = -\frac{1}{N}\sum_i a(v_i,v_i') \otimes x_i, 
\end{equation}
and its norm is bounded by 
\begin{equation}
\|\nabla g(w) \| \leq \frac{1}{N} \sum_i \sqrt{2} \|x_i\| \leq \sqrt{2}.
\end{equation}
Therefore the $L_2$ sensitivity of the minimizer of majority-labeled multiclass
logistic regression is 
\begin{equation}
\frac{\sqrt{2}}{\lambda}.
\end{equation}

\subsection{Weighted ERM}

A natural multiclass extension of the weighted loss (\ref{eq:weighted loss}) is
\begin{equation}\label{eq:weighted loss multiclass}
l^\alpha(w) = \sum_k \alpha^k(x) l(w_k^Tx), 
\end{equation}
where $\alpha^k(x)$ is the unbiased estimate of the probability $P(v=k|x)$.
The corresponding weighted regularized empirical risk is
\begin{eqnarray}
R^\lambda_S(w) &=& \frac{1}{N}\sum_i \sum_k \alpha^k(x_i) l(w_k^Tx) + \frac{\lambda}{2}\|w\|^2\nonumber\\
&=&  \frac{1}{N}\sum_i \sum_k \alpha^k(x_i) [\log(\sum_l e^{w_l^Tx_i}) - w_k^T x_i ]\nonumber \\
&& + \frac{\lambda}{2}\|w\|^2 \nonumber\\
& = & -\frac{1}{N}\sum_i [\sum_k \alpha^k(x_i) w_k^T x_i - \log(\sum_l e^{w_l^Tx_i})]\nonumber\\
&& + \frac{\lambda}{2}\|w\|^2,
\end{eqnarray}
%
and its partial gradient is
\begin{equation}
\nabla_{w_k} R^\lambda_S(w) = 
-\frac{1}{N}\sum_i x_i \left [\alpha^k(x_i)  - \frac{e^{w_k^Tx_i}}{\sum_l e^{w_l^Tx_i}} \right] + \lambda w_k.
\end{equation}

Let $S=\{(x_i,\alpha_i)\}$ and $S'=\{(x_i,\alpha_i')\}$ be two datasets with
the same features but with possibly different labels for all $i=1,...,N$.
Then the partial gradient of the risk difference $g(w)$ is 
\begin{eqnarray}
\nabla_{w_k} g(w) &=& -\frac{1}{N}\sum_i x_i [\alpha^k(x_i) -(\alpha')^k(x_i)]\nonumber\\
&=& -\frac{1}{N}\sum_i x_i b_k(\alpha^k_i,(\alpha')^k_i), 
\end{eqnarray}
where $b_{k}(\alpha^k_i,(\alpha')^k_i) = \alpha^k(x_i)  - (\alpha')^k(x_i)$.
Let $b=[b_1,...,b_K]$ be a $K \times 1$ vector (which depends 
$\alpha_i,\alpha'_i$.)
Note that at most two elements of $b$ can be nonzero (i.e., $\pm 1/M$.)
The gradient can then be rewritten as
\begin{equation}
\nabla g(w) = -\frac{1}{N}\sum_i b(\alpha_i,\alpha_i') \otimes x_i, 
\end{equation}
and its norm is bounded by
\begin{equation}
\|\nabla g(w) \| \leq \frac{1}{N} \sum_i \frac{\sqrt{2}}{M} \|x_i\| \leq \frac{\sqrt{2}}{M}.
\end{equation}
Therefore the $L_2$ sensitivity of the minimizer of the weighted 
multiclass logistic regression is
\begin{equation}
\frac{\sqrt{2}}{M\lambda}.
\end{equation}

\subsection{Parameter averaging}

For the purposes of comparison, we also derive the sensitivity of parameter averaging \cite{Pathak:2010} for multiclass logistic regression.
Let the two neighboring datasets be  
$W = (w_1,w_2,...,w_M)$ and $W' = (w'_1,w'_2,...,w'_M)$, which are
collections of parameters from $M$ parties. 
The corresponding averages for the two sets are
$\bar{w} = \frac{1}{M} \sum_i w_i$ and  $\bar{w}' = \frac{1}{M} \sum_i w'_i$.
Without loss of generality, we assume the parameters $w_1$ and $w_1'$ differ only
for party 1 and $w_i = w'_i$ for others $i=2,...,M$. 
Since $\|\bar{w}-\bar{w}'\| = \frac{1}{M} \|w_1-w_1'\|$, 
the $L_2$ sensitivity is $1/M$ times the sensitivity of the minimizer of the minimizer of a single classifier, when all training samples of party 1 are allowed to change.
Therefore the $L_2$ sensitivity of the average parameters for multiclass 
logistic regression is 
$\frac{2\sqrt{2}}{M\lambda}.$

\end{document} 

